\documentclass[12pt,pdftex,noinfoline]{imsart}

\RequirePackage[OT1]{fontenc}
\usepackage{amsthm,amsmath,amsfonts,natbib,mathtools,amssymb}
\RequirePackage{hypernat}
\usepackage[ruled,section]{algorithm}
\usepackage{algorithmic}
\usepackage{graphicx}
\usepackage{verbatim}
\usepackage{times}
\usepackage[T1]{fontenc}
\usepackage[text={6.0in,8.6in},centering]{geometry}
\usepackage{subfigure}
\usepackage{graphicx}
\usepackage{hyperref}
\usepackage{yhmath}
\usepackage[usenames,dvipsnames]{xcolor}
\definecolor{shadecolor}{gray}{0.9}
\definecolor{shadecolor}{gray}{0.9}
\hypersetup{citecolor=MidnightBlue}
\hypersetup{linkcolor=MidnightBlue}
\hypersetup{urlcolor=MidnightBlue}
\usepackage{enumerate}
\usepackage{fullpage}
\usepackage{tikz}

\renewcommand{\d}{\mathrm{d}}
\newcommand{\E}{\operatorname{\mathbb E}}

\newcommand\numberthis{\addtocounter{equation}{1}\tag{\theequation}}

\DeclareMathOperator*{\argmin}{\arg\min}

\numberwithin{equation}{section}
\newtheorem{theorem}{Theorem}[section]
\newtheorem{lemma}{Lemma}[section]

\theoremstyle{remark}

\newtheorem{remark}{Remark}[section]

\begin{document}

\setlength{\parskip}{0.5em}

\begin{frontmatter}
\title{Distributed Nonparametric Regression under Communication Constraints}
\runtitle{Distributed Nonparametric Regression under Communication Constraints}
\begin{aug}
\vskip10pt
\author{\fnms{Yuancheng} \snm{Zhu${}^{*}$}\ead[label=e1]{zhuyuanc@wharton.upenn.edu}}
\, \and \,
\author{\fnms{John} \snm{Lafferty${}^{\dag}$}\ead[label=e2]{john.lafferty@yale.edu}}
\vskip10pt
\address{
\begin{tabular}{cc}
${}^*$Department of Statistics & ${}^\dag$Department of Statistics and Data Science\\
University of Pennsylvania & Yale University
\end{tabular}
\\[10pt]
\today\\[5pt]
\vskip10pt
}
\end{aug}

\begin{abstract}
This paper studies the problem of nonparametric estimation of a smooth
function with data distributed across multiple machines.  We assume an
independent sample from a white noise model is collected at each
machine, and an estimator of the underlying true function needs to be
constructed at a central machine.  We place limits on the number of
bits that each machine can use to transmit information to the central
machine.  Our results give both asymptotic lower bounds and matching
upper bounds on the statistical risk under various settings. We
identify three regimes, depending on the relationship among the number
of machines, the size of data available at each machine, and the
communication budget.  When the communication budget is small, the
statistical risk depends solely on this communication bottleneck,
regardless of the sample size. In the regime where the communication
budget is large, the classic minimax risk in the non-distributed
estimation setting is recovered.  In an intermediate regime, the
statistical risk depends on both the sample size and the communication
budget. 
\end{abstract}

\vskip20pt 
\end{frontmatter}

\section{Introduction}

Classic statistical theory studies the difficulty of estimation under various models,
and attempts to find the optimal estimation procedures.
Such studies usually assume that all of the collected data
are available to construct the estimators.
In this paper, we study the problem of statistical estimation with data residing at multiple machines. 
Estimation in distributed settings is becoming common in modern data analysis tasks, as the data can be collected or stored at different locations.
In order to obtain an estimate of some statistical functional,
information needs to be gathered and aggregated from the multiple
locations  to form the final estimate. However, the communication
between machines may be limited. For instance,
there may be a communication budget that limits how much information
can be transmitted. In this setting, it is important to understand how the statistical
risk of estimation degrades as the communication budget becomes more
limited.

A similar problem, called the CEO problem, was first studied in the
electrical engineering community from a rate-distortion-theory
perspective \citep{berger1996ceo, viswanathan1997quadratic}.  More
recently, several studies have focused on more specific statistical
tasks and models; see, for example, \citet{zhang2013information,
  shamir2014fundamental, battey2015distributed,
  braverman2016communication, diakonikolas2017communication,
  fan2017distributed, lee2017communication} treating mean estimation,
regression, principal eigenspace estimation, discrete density
estimation and other problems. Most of this existing research focuses
on parametric and discrete models, where the parameter of interest has
a finite dimension.  
While there are also studies of nonparametric problems and models
\citep{zhang2013divide, blanchard2016parallelizing, chang2017distributed, shang2017computational},
the fundamental limits of distributed nonparametric estimation 
are still under-explored. 

In this paper, we consider a fundamental nonparametric estimation task---estimating a smooth function
in the white noise model. 
We assume observation of the random process
\begin{equation}
\d Y(t) = f(t)\d t + \frac{1}{\sqrt{n}}\d W(t),\quad 0\leq t\leq 1, 
\label{eqn:whitenoisemodel}
\end{equation}
where $\frac{1}{\sqrt{n}}$ is the noise level, $W(t)$ is a standard Wiener process,
and $f$ is the underlying function to be estimated.
The white noise model is a centerpiece of nonparametric estimation,
being asymptotically equivalent to nonparametric regression 
and density estimation \citep{brown1996asymptotic,nussbaum1996asymptotic}.
We intentionally express the noise level as $\frac{1}{\sqrt{n}}$ to reflect the connection between
the white noise model and a nonparametric regression problem with $n$ evenly spaced observations. 
We focus on the important case where the regression function
lies in the Sobolev space $\mathcal F(\alpha, c)$ of order $\alpha$ and radius $c$;
the exact definition of this function space is given in the following section.

In a distributed setting, instead of observing a single sample path $Y(t)$, we assume there are $m$ machines, 
each of which observes an independent copy of the stochastic process.
That is, the $j$th machine gets
\[
\d Y_j(t) = f(t)\d t + \frac{1}{\sqrt{n}} \d W_j(t),\quad 0\leq t\leq 1, 
\]
for $j=1,\dots, m$ where $W_j(t)$'s are mutually independent standard Wiener processes.
Furthermore, each machine has a budget of $b$ bits to communicate with a central machine, 
where a final estimate $\hat f$ is formed based
on the messages received from the $m$ machines. 
Specifically, we denote by $\Pi_j$ the message that the $j$th machine
sends to the central estimating machine; 
each $\Pi_j$ can be viewed as a (possibly random) functional of the stochastic process $Y_j(t)$.
In this way, the tuple $(n,m,b)$ defines a problem instance for the function class $\mathcal F(\alpha, c)$.
We use the minimax risk
\begin{align*}
R(n,m,b;\mathcal F(\alpha, c)) = \inf_{\hat f,\, \Pi_{1:m}} \sup_{f\in\mathcal F(\alpha, c)}\,\E \|f-\hat f(\Pi_1, \dots, \Pi_m)\|^2
\end{align*}
to quantify the hardness of distributed estimation of $f$ in the Sobolev space $\mathcal F(\alpha, c)$.

The main contribution of the paper is to identify the following three asymptotic regimes.
\begin{itemize}
\item An \emph{insufficient regime} where $mb \ll n^{\frac{1}{2\alpha+1}}$.
Under this scaling, the total number of bits, $mb$, is insufficient to
preserve the classical, non-distributed, minimax rate of convergence
for the sample size $n$ on a single machine. Therefore, the communication budget becomes the main bottleneck, 
and we have
\[
R(n,m,b;\mathcal F(\alpha, c)) \asymp (mb)^{-2\alpha}.
\]
\item A \emph{sufficient regime} where $b \gg (mn)^{\frac{1}{2\alpha+1}}$.
In this case, the number of bits allowed per machine is relatively large, 
and we have the minimax risk
\[
R(n,m,b;\mathcal F(\alpha, c)) \asymp (mn)^{-\frac{2\alpha}{2\alpha+1}}.
\]
Note that this is also the optimal convergence rate if all the data were available at the central machine. 
\item An \emph{intermediate regime} where $b\lesssim (mn)^{\frac{1}{2\alpha+1}}$ and $mb\gtrsim n^{\frac{1}{2\alpha+1}}$.
In this regime, the minimax risk depends on all three parameters, and
scales according to
\[
R(n,m,b;\mathcal F(\alpha, c)) \asymp (mnb)^{-\frac{\alpha}{\alpha+1}}.
\]
\end{itemize}
Together, these three regimes give a sharp characterization of the statistical
behavior of distributed nonparametric estimation for the Sobolev space
${\mathcal F}(\alpha,c)$ under communication constraints, covering the
full range of parameters and problem settings.  The Bayesian framework
adopted in this paper to establish the lower bounds is different from
the techniques used in previous work, which typically rely on Fano's
lemma and the strong data processing inequality.  Finally, we note
that an essentially equivalent set of minimax convergence rates is obtained
in a simultaneously and independently written paper
by \citet{szabo2018adaptive}.

The paper is organized as follows. In the next section, we explain our
notation and give a brief introduction of nonparametric estimation
over a Sobolev space for the usual non-distributed setting and a distributed setting. 
In Section \ref{sec:mainresults}, we state our main results on the risk of distributed nonparametric estimation with communication constraints. 
We outline the proof strategy for the lower bounds in Section
\ref{subsec:lowerbounds}, deferring some of the technical details and
proofs to the supplementary material.
In Section \ref{sec:achievability}, we show achievability of the lower
bounds by a particular distributed protocol and estimator. 
We conclude the paper with a discussion of possible directions for future work. 

\section{Problem formulation}

The Sobolev space of order $\alpha$ and radius $c$ is defined by 
\begin{align*}
\mathcal F(\alpha, c) & = \bigg\{f: f^{(\alpha-1)} \text{ is absolutely continuous}, \int_0^1(f^{(\alpha)}(t))^2)\d t\leq c^2,\text{ and }f\in[0,1] \to \mathbb R\bigg\}.
\end{align*}
Intuitively, it is a space of functions having a certain degree of smoothness. 
The periodic Sobolev space is defined by
\begin{align*}
\tilde{\mathcal F}(\alpha, c) = & F(\alpha, c) \;\bigcap \bigg\{ f^{(j)}(0) = f^{(j)}(1),\, j = 0, 1, \dots, \alpha - 1\bigg\}.
\end{align*}

The white noise model \eqref{eqn:whitenoisemodel} can be reformulated 
in terms of an infinite Gaussian sequence model. 
Let $(\varphi_i)_{i = 1}^\infty$ be the trigonometric basis, and let
\[
\theta_i = \int_0^1 \varphi_i(t)f(t)\d t, \quad i = 1,2,\dots
\]
be the Fourier coefficients.
It is known that $f$ belongs to $\tilde{\mathcal F}(\alpha, c)$ if and
only if the sequence $\theta$ belongs to the Sobolev ellipsoid
$\Theta(\alpha, c)$, defined as
\[
\Theta(\alpha, c) = \left\{\theta:\sum_{i=1}^\infty a_i^2 \theta_i^2\leq \frac{c^2}{\pi^{2\alpha}}\right\}
\]
where
\[
a_i = \begin{cases}
i^\alpha & \text{if $i$ is even}\\
(i-1)^\alpha & \text{if $i$ is odd.}
\end{cases}
\]
To ease the analysis, we will assume $a_i = i^\alpha$ and use $\tilde
c^2$ in the place of $\frac{c^2}{\pi^{2\alpha}}$.
Expanding the observed process $Y(t)$ in terms of the same basis we obtain the Gaussian sequence
\[
X_i = \int_0^1 \varphi_i(t) \,\d Y(t) \sim N\left(\theta_i, 1/n\right).
\]
Given an estimator $\hat\theta$ for $\theta$, we can formulate a corresponding estimator for $f$ by
\[
\hat f(t) = \sum_{i = 1}^\infty \hat \theta_i \varphi_i(t),
\]
and the squared errors satisfy $\|\hat\theta - \theta\|^2 = \|\hat f - f\|^2$.
In this way, estimating the function $f$ in the white noise model is equivalent to estimating the means $\theta$ in the Gaussian sequence model. 

The minimax risk of estimating $f$ over the periodic Sobolev space is defined as
\[
R(n; \tilde{\mathcal F}(\alpha, c)) = \inf_{\hat f} \sup_{f\in\mathcal F(\alpha, c)} \E \|\hat f - f\|^2,
\]
which, as just shown, is equal to the minimax risk of estimating $\theta$ over the Sobolev ellipsoid in the corresponding Gaussian sequence model,
\[
R(n; \Theta(\alpha, \tilde c)) = \inf_{\hat \theta} \sup_{\theta \in \Theta(\alpha, \tilde c)} \E \|\hat \theta - \theta\|^2.
\]
It is known \citep{tsybakov2008introduction} that the asymptotic
minimax risk scales according to
\[
R(n;\tilde{\mathcal F}(\alpha, c)) = R(n;\Theta(\alpha, c))\asymp n^{-\frac{2\alpha}{2\alpha+1}}
\]
as $n\to\infty$.

In a distributed setting, we suppose there are $m$ machines, and the $j$th machine independently observes $Y_j(t)$ such that
\[
\d Y_j(t) = f(t)\d t + \frac{1}{\sqrt{n}} \d W_j(t) \quad 0\leq t \leq 1
\]
for $j=1,\dots, m$.
Equivalently, if we express this in terms of the Gaussian sequence
model, the $j$th machine observes data
\[
X_{ij} \sim N(\theta_i, 1/n), \quad i = 1, 2, \dots .
\]
We further assume there is a central machine where a final estimator needs to be calculated based on messages received from the $m$ local machines. 
Local machine $j$ sends a message of length $b_j$ bits to the central machine;
we denote this message by $\Pi_j$. Then
$\Pi_j = \Pi_j(X_{1j},X_{2j},\dots)$ can be viewed as a (possibly random) mapping from $\mathbb R^\infty$ to $\{1, 2,\dots, 2^{b_j}\}$.
The final estimator $\hat \theta$ is then a functional of the collection of messages.
The mechanism can be summarized by the following diagram:
\small
\[
f\longrightarrow\begin{Bmatrix}
Y_1(t)\longrightarrow X_{11},\dots, X_{n1}\overset{b_1}{\longrightarrow} \Pi_1\\
Y_2(t)\longrightarrow X_{12},\dots, X_{n2}\overset{b_2}{\longrightarrow} \Pi_2\\
\vdots\\
Y_m(t)\longrightarrow X_{1m},\dots, X_{nm}\overset{b_m}{\longrightarrow} \Pi_m\\
\end{Bmatrix}
\longrightarrow \hat\theta\longrightarrow \hat f.
\]
\normalsize
Suppose that the communication is restricted by one of two types of constraints:
An \emph{individual constraint}, where $b_j\leq b$, for each $j = 1,
\dots, m$ and a given budget $b$,
and a \emph{sum constraint}, where $\sum_{j = 1}^m b_j\leq mb$.
We call the set of mappings $\Pi_1,\dots, \Pi_m$ and $\hat\theta$ a
{\it distributed protocol}, 
and denote by $\Gamma_{\text{ind}}(m,b)$ and
$\Gamma_{\text{sum}}(m,b)$ the collection of all such protocols,
operating under the individual constraint and the sum constraint,
respectively. 

We note here that for simplicity we consider only one round of communication.
A variant is to allow multiple rounds of communication, for which
the local machines can get access to a ``blackboard'' where the central machine broadcasts
information back to the distributed nodes.

The minimax risk of the distributed estimation problem under the communication constraint is defined by
\begin{equation}
\label{eqn:distributedminimaxrisk}
\begin{aligned}
R(n,m,b;\Theta(\alpha,c))
= \inf_{(\Pi_1,\dots, \Pi_m,\hat\theta)\in\Gamma(m,b)} \sup_{\theta\in\Theta(\alpha, c)} \E\|\hat\theta(\Pi_1,\dots, \Pi_m) - \theta\|^2.
\end{aligned}
\end{equation}
Here $\Gamma$ represents either $\Gamma_{\text{ind}}$ or $\Gamma_{\text{sum}}$.
In fact, it will be clear that the minimax risks under the two types of constraints are asymptotically equivalent. 

\section{Lower bounds for distributed estimation}
\label{sec:mainresults}

In what follows, we will work in an asymptotic regime where the tuple $(n,m,b)$ goes to infinity while satisfying some relationships,
and show how the minimax risk for the distributed estimation problem scales accordingly. 
The main result can be summarized in the following theorem. 

\begin{theorem} Let $R(n,m,b;\Theta(\alpha,c))$ be defined as in \eqref{eqn:distributedminimaxrisk}
with $\Gamma = \Gamma_{\text{sum}}$
\label{thm:main}
\begin{enumerate}
\item If $b(mn)^{-\frac{1}{2\alpha+1}}\to\infty$, then
\[
\liminf_{mn\to\infty}\ (mn)^{\frac{2\alpha}{2\alpha+1}}R(n,m,b;\Theta(\alpha,c)) \geq C.
\]
\item If $b(mn)^{-\frac{1}{2\alpha+1}} = O(1)$ and $mbn^{-\frac{1}{2\alpha+1}}\to\infty$, then
\[
\liminf_{mn\to\infty}\ (mnb)^{\frac{\alpha}{\alpha+1}}R(n,m,b;\Theta(\alpha,c)) \geq C.
\]
\item If $mbn^{-\frac{1}{2\alpha+1}} = O(1)$, then
\[
\liminf_{mn\to\infty}\ (mb)^{2\alpha}R(n,m,b;\Theta(\alpha,c)) \geq C.
\]
\end{enumerate}
\end{theorem}
\begin{remark}
The lower bounds are valid for both the sum constraint
and the individual constraint.
In fact, the individual constraint is more stringent than the sum constraint,
so in terms of lower bounds, it suffices to prove it for the sum constraint. 
\end{remark}
\begin{remark}
To put the result more concisely, we can write
\begin{align*}
R(n,m,b;\Theta(\alpha,c)) 
\gtrsim
\begin{cases}
(mn)^{-\frac{2\alpha}{2\alpha+1}} & \text{if }b(mn)^{-\frac{1}{2\alpha+1}}\to \infty \\
(mbn)^{-\frac{\alpha}{\alpha+1}} & 
\begin{aligned}
\text{if } b(mn)^{-\frac{1}{2\alpha+1}} =O(1)\\
\text{and }mbn^{-\frac{1}{2\alpha+1}}\to\infty
\end{aligned}\\
(mb)^{-2\alpha} & \text{if }mbn^{-\frac{1}{2\alpha+1}} = O(1)
\end{cases}.
\end{align*}
There are multiple ways to interpret this main result and here we illustrate one of the many possibilities.
Fixing $m$ and $b$, and viewing the minimax risk as a function of $n$, the sample size on each machine,
we have
\begin{align*}
R(n) \gtrsim\begin{cases}
n^{-\frac{2\alpha}{2\alpha+1}} m^{\frac{2\alpha}{2\alpha+1}} & \text{if } n\lesssim \frac{b^{2\alpha+1}}{m} \\
n^{-\frac{\alpha}{\alpha+1}} (mb)^{-\frac{\alpha}{\alpha+1}} & \text{if } \frac{b^{2\alpha+1}}{m}\ll n \ll (mb)^{2\alpha+1} \\
(mb)^{-2\alpha} & \text{if } n \gtrsim (mb)^{2\alpha+1} \\
\end{cases}.
\end{align*}
This indicates that when the configuration of machines and communication budget stay the same, 
as we increase the sample size at each machine,
the risk starts to decay at the optimal rate with exponent $-\frac{2\alpha}{2\alpha+1}$.
Once the sample size is large enough, the convergence rate slows down to an exponent $-\frac{\alpha}{\alpha+1}$.
Eventually, the sample size exceeds a threshold, 
beyond which any further increase won't decrease the risk due to the communication constraint. 
\end{remark}
\begin{remark}
This work can be viewed as a natural generalization of \citet{zhu2017quantized},
where the authors consider estimation over a Sobolev space with a single remote machine and communication constraints. 
Specifically, by setting $m=1$ we recover the main results in \citet{zhu2017quantized} up to some constant factor. 
However, with more than one machine, it is non-trivial to uncover the minimax convergence rate, 
especially in the intermediate regime. 

\end{remark}

\subsection{Proof of the lower bounds}\label{subsec:lowerbounds}
We now proceed to outline the proof of the lower bounds in Theorem
\ref{thm:main}.  Most existing results rely on Fano's lemma and the
strong data processing inequality \citep{zhang2013information,
  braverman2016communication}.  An
extension of this information-theoretic approach is used
by \citet{szabo2018adaptive} in the nonparametric
setting to obtain essentially the same lower bounds as we establish here.
However, we develop the Bayesian framework for deriving minimax lower
bounds \citep{johnstone2017gaussian}, circumventing the need for both
Fano's lemma and the strong data processing inequality, and
associating the lower bounds with the solution of an optimization
problem.

We consider a prior distribution $\pi(\theta)$ asymptotically supported on the parameter space $\Theta$.
For any estimator $\hat\theta$ that follows the distributed protocol, we have
\begin{equation}\label{eqn:supgtrintegral}
\sup_{\theta\in\Theta}\E_\theta \|\hat\theta - \theta\|^2 \gtrsim \int_{\Theta}\E_\theta\|\hat\theta - \theta\|^2 \d\pi(\theta).
\end{equation}
That is, the worst-case risk associated with $\hat\theta$ is bounded
from below by the integrated risk. 
We specifically consider the Gaussian prior distribution
$\theta_i \sim N(0, \sigma_i^2)$ for $i=1,\dots,\ell$, and $\mathbb P(\theta_i = 0) = 1$ for $i = \ell+1,\dots$,
where the sequence $\sigma_i$ satisfies $\sum_{i=1}^\ell i^{2\alpha}\sigma_i^2 \leq c^2$.
We make \eqref{eqn:supgtrintegral} clear in the following lemma, whose proof can be found in the supplementary material. 
\begin{lemma}\label{lem:supgtrintegral}
Suppose that a sequence of Gaussian prior distributions for $\theta$ and estimator $\hat\theta$ satisfy
\begin{equation}\label{eqn:condition}
\frac{\sum_{i=1}^\ell i^{2\alpha}\sigma_i^2}{\max_{1\leq i\leq \ell}i^{2\alpha}\sigma_i^2} = O(\ell)
\text{ and }
\int_{\Theta}\E_\theta[\|\hat\theta - \theta\|^2] \d\pi(\theta) = O(\ell^{\delta})
\end{equation}
for some $\delta>0$ as $\ell\to\infty$. Then
\[
\sup_{\theta\in\Theta}\E_\theta \|\hat\theta - \theta\|^2 \geq \int_{\Theta}\E_\theta[\|\hat\theta - \theta\|^2] \d\pi(\theta)\cdot (1+o(1)).
\]
\end{lemma}

The next step is to lower bound the integrated risk $\int_{\Theta}\E_\theta[\|\hat\theta - \theta\|^2] \d\pi(\theta)$.
Lemma \ref{lem:lowerboundwithprior} is derived from a result that appears in \citep{wang2010sum};
for completeness we include the proof in the supplementary material. 
\begin{lemma}\label{lem:lowerboundwithprior}
Suppose $\theta_i\sim N(0,\sigma_i^2)$ and $X_{ij}\sim N(\theta_i, \varepsilon^2)$ for $i=1,\dots, \ell$ and $j=1,\dots, m$.
Let $\Pi_j: \mathbb R^\ell\to\{1,\dots, M_j\}$ be a (random) mapping, which takes up to $M_j$ different values. 
Let $\hat\theta: \{1,\dots, M_1\}\times\cdots\times\{1,\dots, M_m\}\to \mathbb R^\ell$ be an estimator 
based on the messages created by $\Pi_1,\dots, \Pi_m$. 
Under the constraint that $\frac{1}{m}\sum_{j=1}^m \log M_j \leq b$, 
$\E\|\hat\theta - \theta\|^2$ can be lower bounded by the value of the following optimization problem
\begin{equation}\label{eqn:optimization}
\begin{aligned}
L(m,b,\varepsilon;\sigma) \triangleq
\min_{d_i,\ i = 1,\dots,\ell} &\ \sum_{i=1}^\ell d_i\\
\mathrm{s.t.} &\ \sum_{i=1}^\ell\left(\frac{1}{2}\log\frac{\sigma_i^2}{d_i} + \frac{m}{2}\log\frac{\frac{m}{\varepsilon^2}}{\frac{1}{\sigma_i^2} + \frac{m}{\varepsilon^2}-\frac{1}{d_i}}\right)\leq mb\\
&\ \frac{\sigma_i^2\frac{\varepsilon^2}{m}}{\sigma_i^2+\frac{\varepsilon^2}{m}} \leq d_i \leq \sigma_i^2 \mathrm{\ for\ } i = 1, \dots, \ell.
\end{aligned}
\end{equation}
\end{lemma}

Combining the Lemma \ref{lem:supgtrintegral} and \ref{lem:lowerboundwithprior},
we have the following asymptotic lower bound
\[
R(m,b,n;\Theta(\alpha, c)) \gtrsim L(m,b,n^{-\frac{1}{2}};\sigma)
\]
for sequences $\sigma_i$ satisfying $\sum_{i=1}^\ell i^{2\alpha}\sigma_i^2\leq \tilde c^2$ and 
$\frac{\sum_{i=1}^\ell i^{2\alpha}\sigma_i^2}{\max_{1\leq i\leq \ell}i^{2\alpha}\sigma_i^2} = O(\ell)$
as $\ell\to\infty$.

Next, based on the optimization problem formulated above, 
we work under three different regimes,
and derive three forms of lower bounds of the minimax risk.
The key is to choose appropriate sequences of prior variances $\sigma_i^2$ for different regimes,
as we shall illustrate. 

\begin{enumerate}
\item Suppose that $d_1,\dots, d_\ell$ is a feasible solution to the problem \eqref{eqn:optimization}. 
Using the first constraint, we have
\begin{align*}
mb & \geq \sum_{i=1}^\ell\left(\frac{1}{2}\log\frac{\sigma_i^2}{d_i} + \frac{m}{2}\log\frac{\frac{m}{\varepsilon^2}}{\frac{1}{\sigma_i^2} + \frac{m}{\varepsilon^2}-\frac{1}{d_i}}\right) \\
& \geq \sum_{i=1}^\ell\frac{1}{2}\log\frac{\sigma_i^2}{d_i}\\
& \geq \sum_{i=1}^\ell\frac{1}{2}\log\sigma_i^2 - \frac{\ell}{2}\log \left(\frac{1}{\ell}\sum_{i=1}^\ell d_i\right)
\end{align*}
where we have used Jensen's inequality. 
Therefore,
\begin{align*}
\sum_{i=1}^\ell d_i &\geq \ell\exp\left( \frac{1}{\ell}\sum_{i=1}^\ell\log\sigma_i^2 - \frac{2mb}{\ell} \right)\\
&= \ell\left(\prod_{i=1}^\ell \sigma_i^2\right)^{\frac{1}{\ell}}e^{-\frac{2mb}{\ell}}.
\end{align*}

Consider an asymptotic regime where $mbn^{-\frac{1}{2\alpha+1}} = O(1)$,
and pick a sequence of corresponding prior distributions with $\ell = \gamma mb$ for some constant $\gamma$
and $\sigma_i^2 = \frac{\tilde c^2}{i^{2\alpha}\ell}$ for $i = 1,\dots, \ell$.
Note that this choice satisfies condition \eqref{eqn:condition}.
With such a choice of the prior distribution, we have
\begin{align*}
\sum_{i=1}^\ell d_i & \geq \ell\left(\prod_{i=1}^\ell \frac{\tilde c^2}{i^{2\alpha}\ell}\right)^{\frac{1}{\ell}}e^{-\frac{2mb}{\ell}}\\
& = \tilde c^2e^{-\frac{2mb}{\ell}}(\ell!)^{-\frac{2\alpha}{\ell}}\\
& \geq \tilde c^2e^{-\frac{2mb}{\ell}} (e\ell^{\ell+\frac{1}{2}}e^{\ell})^{-\frac{2\alpha}{\ell}} \\
& \geq \tilde c^2e^{-4\alpha}\cdot e^{-\frac{2mb}{\ell}}\ell^{-2\alpha - \frac{\alpha}{\ell}}\\
& \sim (mb)^{-2\alpha}.
\end{align*}

\item Again suppose that $d_1,\dots, d_\ell$ is a feasible solution to the problem \eqref{eqn:optimization}. 
This time we take another viewpoint of the first constraint
\begin{align*}
mb  &\geq \sum_{i=1}^\ell\left(\frac{1}{2}\log\frac{\sigma_i^2}{d_i} + \frac{m}{2}\log\frac{\frac{m}{\varepsilon^2}}{\frac{1}{\sigma_i^2} + \frac{m}{\varepsilon^2}-\frac{1}{d_i}}\right) \\
&\geq \sum_{i=1}^\ell\frac{m}{2}\log\frac{\frac{m}{\varepsilon^2}}{\frac{1}{\sigma_i^2} + \frac{m}{\varepsilon^2}-\frac{1}{d_i}}.
\end{align*}
To minimize $\sum_{i=1}^\ell d_i$ under the constraint that 
$\sum_{i=1}^\ell\frac{1}{2}\log\frac{\frac{m}{\varepsilon^2}}{\frac{1}{\sigma_i^2} + \frac{m}{\varepsilon^2}-\frac{1}{d_i}}\leq b$,
we write the Lagrangian
\[
L = \sum_{i=1}^\ell d_i + \lambda\left(\sum_{i=1}^\ell\frac{1}{2}\log\frac{\frac{m}{\varepsilon^2}}{\frac{1}{\sigma_i^2} + \frac{m}{\varepsilon^2}-\frac{1}{d_i}} - b\right)
\]
and set
\[
0 = \frac{\partial L}{\partial d_i} = 1 - \frac{\lambda}{2}\frac{1}{\frac{1}{\sigma_i^2} + \frac{m}{\varepsilon^2} - \frac{1}{d_i}}.
\]
Solving this gives us that
\begin{align*}
\sum_{i = 1}^\ell d_i & \geq \sum_{i=1}^\ell \frac{1}{\frac{1}{\sigma_i^2} + \frac{m}{\varepsilon^2} \left(1-e^{-\frac{2b}{\ell}}\right)}\\
& \geq \sum_{i=1}^\ell \frac{1}{\frac{1}{\sigma_i^2} + \frac{2mb}{\varepsilon^2\ell} }.
\end{align*}
This time, consider a regime where $b(mn)^{-\frac{1}{2\alpha+1}} =O(1)$
and $mbn^{-\frac{1}{2\alpha+1}}\to\infty$.
Pick a sequence of corresponding prior distributions with $\ell = (\gamma mbn)^{\frac{1}{2\alpha+2}}$ for some constant $\gamma$
and $\sigma_i^2 = \frac{\tilde c^2}{\sum_{i=1}^\ell i^{2\alpha}}$ for $i = 1,\dots, \ell$,
which satisfies condition \eqref{eqn:condition}.
With this choice and replacing $\varepsilon^2$ by $\frac{1}{n}$, we have
\begin{align*}
\sum_{i = 1}^\ell d_i & \geq \sum_{i=1}^\ell \frac{1}{\frac{\sum_{i=1}^\ell i^{2\alpha}}{\tilde c^2} + \frac{2mbn}{\ell} }\\
& \geq \frac{\ell}{\frac{(\ell+1)^{2\alpha+1}}{\tilde c^2(2\alpha+1)} + \frac{2mbn}{\ell} }\\
& = \frac{\tilde c^2(2\alpha+1)}{2\tilde c^2(2\alpha+1)+1}\ell^{-2\alpha}(1+o(1))\\
&\sim (mbn)^{-\frac{\alpha}{\alpha+1}}.
\end{align*}

\item For the last regime where $b(mn)^{-\frac{1}{2\alpha+1}}\to \infty$,
we use the constraint that $d_i\geq\frac{\sigma_i^2\frac{\varepsilon^2}{m}}{\sigma_i^2+\frac{\varepsilon^2}{m}}$
and write
\[
\sum_{i=1}^\ell d_i \geq \sum_{i=1}^\ell \frac{\sigma_i^2\frac{\varepsilon^2}{m}}{\sigma_i^2+\frac{\varepsilon^2}{m}}
= \sum_{i=1}^\ell \frac{\sigma_i^2\frac{1}{mn}}{\sigma_i^2+\frac{1}{mn}}.
\]
Let $\ell = (\gamma mn)^{\frac{1}{2\alpha+1}}$ and $\sigma_i^2 = \frac{\tilde c^2}{\sum_{i=1}^\ell i^{2\alpha}}$
satisfying \eqref{eqn:condition},
and we have
\begin{align*}
\sum_{i=1}^\ell d_i & \geq \sum_{i=1}^\ell \frac{\frac{\tilde c^2}{\sum_{i=1}^\ell i^{2\alpha}}\frac{1}{mn}}{\frac{\tilde c^2}{\sum_{i=1}^\ell i^{2\alpha}}+\frac{1}{mn}}\\
& \geq \sum_{i=1}^\ell \frac{\frac{\ell^{2\alpha+1}}{2\alpha+1}\frac{1}{mn}}{\frac{\ell^{2\alpha+1}}{2\alpha+1}+\frac{1}{mn}}\\
& \sim (mn)^{-\frac{2\alpha}{2\alpha+1}}.
\end{align*}
Thus, combining the previous three scenarios, we conclude the lower bound in \ref{thm:main}.

\end{enumerate}

\section{Achievability}\label{sec:achievability}

In this section, we describe how the lower bound can be achieved
through the use of a certain distributed protocol. 
Unlike for the lower bound, we shall work under the individual
constraint on the communication budget, instead of the sum constraint. 
However, a protocol satisfying the individual constraint automatically satisfies the sum constraint. 

\subsection{High-level idea}

In nonparametric estimation theory, it is known that for the Gaussian sequence model
$X_i\sim N\left(\theta, \frac{1}{n}\right)$ for $i = 1,\dots,\infty$ with $\theta\in\Theta(\alpha, c)$,
the optimal scaling of the $\ell_2$ risk is
$n^{-\frac{2\alpha}{2\alpha+1}}$, and this can be achieved by
truncating the sequence at $i = O(n^{\frac{1}{2\alpha+1}})$. That is, the estimator
\[
\hat\theta = \begin{cases}
X_i & \text{if } i \leq n^{\frac{1}{2\alpha+1}}\\
0 & \text{if } i > n^{\frac{1}{2\alpha+1}}\\
\end{cases}
\]
has worst-case risk $\sup_{\theta\in\Theta(\alpha, c)}\E_\theta \|\hat\theta - \theta\|^2 \asymp n^{-\frac{2\alpha}{2\alpha+1}}$.
We are going to build on this simple but rate-optimal estimator in our distributed protocol.
But before carefully defining and analyzing the protocol,
we first give a high-level idea of how it is designed. 

In our distributed setting, we have a total budget of $mb$ bits to
communicate from the local machines to the central
machine, which means that we can transmit $O(mb)$ random variables to a certain degree of precision. 

In the first regime where we have $mb \lesssim n^{\frac{1}{2\alpha+1}}$, 
the communication budget is so small that 
the total number of bits is smaller than the effective dimension for the noise level $1/n$.
In this case, we let each machine transmit information regarding a unique set of $O(b)$ components of $\theta$.
Thus, at the central machine, we can decode and obtain information about the first $O(mb)$ components of $\theta_i$.
This is equivalent to truncating a centralized Gaussian sequence at $i = O(mb)$, and gives us a convergence rate of $(mb)^{-2\alpha}$.

In the second regime ($b \ll (mn)^{\frac{1}{2\alpha+1}}$ and $mb\gg n^{\frac{1}{2\alpha+1}}$), 
we have a larger budget at our disposal,
and can thus afford to transmit more than one random variable
containing information about $\theta_i$. 
Suppose that for a specific $i$ we quantize and transmit $X_{ij}$ for $k$ different values of $j$, namely at $k$ different machines.
The budget of $O(mb)$ random variables will allow us to acquire
information about the first $O(\frac{mb}{k})$ components of $\theta$.
When aggregating at the central machine, 
we have $Z_{i}\sim N\left(\theta_i, \frac{1}{nk}\right)$ for $i = 1,\dots, O\left(\frac{mb}{k}\right)$,
and no information about $\theta_i$ for $i\geq O\left(\frac{mb}{k}\right)$.
Now consider the effect of choosing different values of $k$.
In choosing a smaller $k$, we will be able to estimate more components of $\theta$,
but each at a lower accuracy. 
On the other hand, a larger $k$ leads to fewer components being estimated, but with smaller error. 
We know from nonparametric estimation theory that the tradeoff is optimized when
$(kn)^{\frac{1}{2\alpha+1}} \asymp \frac{mb}{k}$.
This gives us the optimal choice $k\asymp (mb)^{\frac{2\alpha+1}{2\alpha+2}}n^{-\frac{1}{2\alpha+2}}$,
with risk scaling as $(mbn)^{-\frac{\alpha}{\alpha+1}}$.

In the last regime, we have $b\gtrsim (mn)^{\frac{1}{2\alpha+1}}$.
In this case, the number of bits available at each machine is larger than the effective dimension associated with the global noise level $\frac{1}{mn}$.
We simply quantize and transmit the first
$O(mn)^{\frac{1}{2\alpha+1}}$ of $X_{ij}$ from each machine to the
central machine, where we decode and simply average the received random variables.

\subsection{Algorithm}

\begin{figure*}[th]
\begin{center}
\input{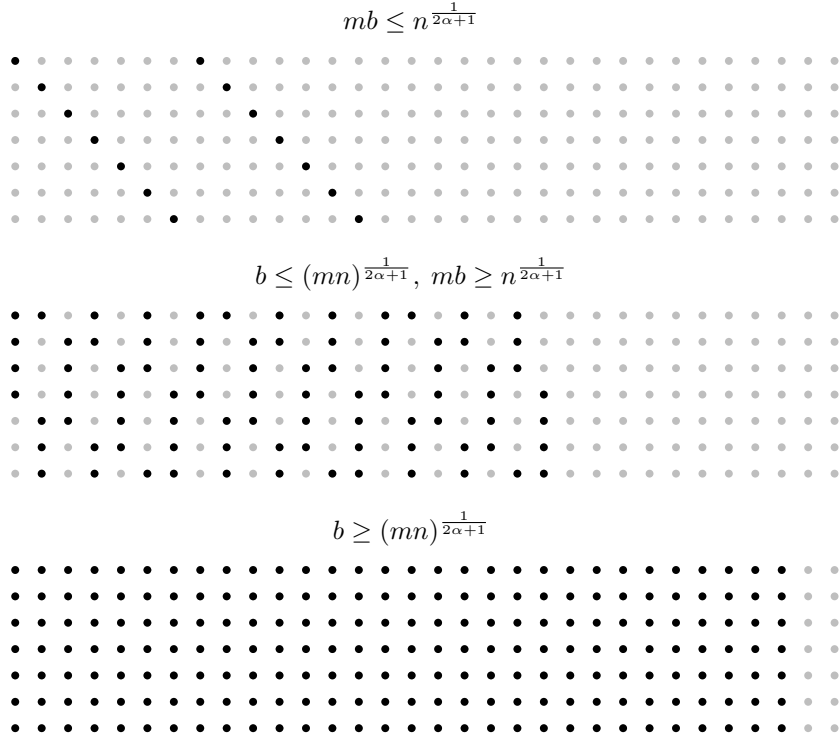}
\caption{Allocation of communication budget for the three regimes. 
Each dot represents a random variable $X_{ij}$. 
The $j$th row represents the random variables on the $j$th machine,
and the $i$th dot in that row is for the random variable $X_{ij}\sim N(\theta_i,\frac{1}{n})$.
If a dot is colored black, it means that the random variable is
quantized and transmitted to the central machine;
otherwise, we don't spend any communication budget on it. 
In the first regime, each $\theta_i$ is only estimated on at most one machine,
while in the second regime, it is estimated on multiple but not all machines.
In the last regime, we quantize and transmit all random variables associated with $\theta_i$ on the $m$ machines,
before truncating at some position. 
}
\label{fig:algorithm}
\end{center}
\end{figure*}

First we state a lemma describing and analyzing a simple scalar quantization method. 
\begin{lemma}\label{lem:randomquantization}
Suppose that $X$ is a random variable supported on $[-c,c]$,
and that $U\sim \mathrm{Unif}(0,\delta)$ independently, for some constant $\delta>0$.
Let $G(u,\delta) = \{u+i\delta: i = 0, \pm 1, \pm 2, \dots\}$ be a grid of points with base point $u$ and skip $\delta$. 
Define 
\[
q(x;u,\delta) = \argmin_{g\in G(u,\delta)}|x-g|.
\]
Let $E = q(X; U,\delta)-X$. Then $X$ and $E$ are independent, and $E\sim\mathrm{Unif}(-\frac{\delta}{2}, \frac{\delta}{2})$.
\end{lemma}

\begin{proof}
Let us condition on the event that $X=x$. 
We have for $\epsilon\in(-\frac{\delta}{2}, \frac{\delta}{2})$
\begin{align*}
&\mathbb P(E \in (\epsilon, \epsilon+\d \epsilon)\,|\,X = x) \\
& = \mathbb P(q(X)\in (x+\epsilon,x+\epsilon+\d \epsilon)\,|\, X = x) \\
& = \mathbb P(U\in (x+\epsilon - \delta \lfloor (x+\epsilon)/\delta \rfloor, \\
&\quad\quad\quad x+\epsilon - \delta \lfloor (x+\epsilon)/\delta \rfloor +\d \epsilon)\,|\, X = x) \\
& = \mathbb P(U\in (x+\epsilon - \delta \lfloor (x+\epsilon)/\delta \rfloor,\\
&\quad\quad\quad x+\epsilon - \delta \lfloor (x+\epsilon)/\delta \rfloor +\d \epsilon)) \\
& = \frac{d\epsilon}{\delta}.
\end{align*}
We thus conclude that $E|X\sim\mathrm{Unif}(-\frac{\delta}{2}, \frac{\delta}{2})$, 
and therefore $E$ and $X$ are independent. 
\end{proof}
By this lemma, we know that with a public key for randomness, 
we can transmit a random variable $X$ supported on $(-c, c)$ using $\log_2\frac{2c}{\delta}$ bits,
so that the central machine receives $X+E$ with $E\sim\mathrm{Unif}(-\frac{\delta}{2}, \frac{\delta}{2})$ and independent to $X$.
We are now ready to describe the algorithm of estimating $\theta$. 

\noindent\textbf{Algorithm}
\vspace{-1em}
\begin{enumerate}
\item Input
\begin{itemize}
\item $\alpha$: order of the Sobolev space. 
\item $c$: radius of the Sobolev space.
\item $X_{ij}$: independent $N\left(\theta_i,\frac{1}{n}\right)$ r.v.\ for $i = 1, \dots, \infty$ at machine $j$ for $j = 1, \dots, m$.
\item $b$: number of bits for communication at each machine.
\end{itemize}
Calculate 
\begin{itemize}
\item $\delta = \max\left\{(mb)^{-\frac{2\alpha+1}{2}}, n^{-\frac{1}{2}}\right\}$.
\item $b_0 = \log_2\delta$, $\tilde b = \lfloor b/b_0\rfloor$.
\item $k = \left(\lfloor (mb)^{\frac{2\alpha+1}{2\alpha+2}}n^{-\frac{1}{2\alpha+2}}\rfloor \vee 1\right) \wedge m$.
\end{itemize}
\item At the $j$th machine (for $j = 1, \dots, m$), let $I_j = \left\{\lceil (ms+j)/k\rceil: s = 0, \dots, \tilde b-1\right\}$.
\begin{enumerate}
\item Generate a random seed shared with the central machine. 
\item For $i\in I_j$, generate $U_{ij}\sim\mathrm{Unif}(0,\delta)$  independently based on the seed. 
\item For $i\in I_j$, winsorize $X_{ij}$ at $[-c, c]$ and quantize 
\[
\tilde X_{ij} = q\left((X_{ij}\wedge c)\vee(-c); U_{ij}, \delta\right).
\] 
\item Transmit the quantized random variables $\left\{\tilde X_{ij}: i\in I_j\right\}$ to the central machine using $\lfloor b/b_0 \rfloor b_0\leq b$ bits. 
\end{enumerate}
\item At the central machine, decode the messages and construct the estimator
\begin{align*}
\hat\theta_i =\begin{cases}
\frac{1}{k}\sum_{j:\,i\in I_j}\tilde X_{ij} &\text{if } i \leq \lfloor m\tilde b/k\rfloor \wedge (mn)^{\frac{1}{2\alpha+1}}\\
0 &\text{otherwise}
\end{cases}.
\end{align*}
\end{enumerate}

A graphical illustration of the algorithm is shown in Figure \ref{fig:algorithm}.
We must also note that while the algorithm is rate optimal, it is not adaptive,
in the sense that it requires knowledge of the parameter $\alpha$.

\subsection{Analysis}
We now analyze the statistical risk associated with the algorithm described in the previous section. 
Suppose that $\theta\in\Theta(\alpha, c)$. 
First notice that the winsorization in Step 2(c) makes $X_{ij}$ bounded prior to quantization
and it only decreases the risk. 
Write $i^* = \lfloor m\tilde b/k\rfloor \wedge (mn)^{\frac{1}{2\alpha+1}}$.
The risk of the final estimator satisfies
\begin{align*}
\E_\theta\left[\|\hat\theta - \theta\|^2\right] 
= \sum_{i = 1}^{i^*}\E_\theta\bigg[\bigg(\frac1k\sum_{j:\,i\in I_j}\tilde X_{ij} - \theta_i\bigg)^2\bigg] + \sum_{i = i^* + 1}^{\infty}\theta_i^2
\end{align*}
where
\begin{align*}
\sum_{i = 1}^{i^*}\E_\theta\bigg[\bigg(\frac1k\sum_{j:\,i\in I_j}\tilde X_{ij} - \theta_i\bigg)^2\bigg] 
& = \sum_{i = 1}^{i^*}\E_\theta\bigg[\bigg(\frac1k\sum_{j:\,i\in I_j}\bigg(X_{ij} + E_{ij}\bigg) - \theta_i\bigg)^2\bigg]  \\
& = \sum_{i = 1}^{i^*}\E_\theta\bigg[\bigg(\frac1k\sum_{j:\,i\in I_j}X_{ij} - \theta_i \bigg)^2\bigg]
+ \E\bigg[\bigg(\frac1k\sum_{j:\,i\in I_j}E_{ij}\bigg)^2\bigg] \\
& \leq \frac{\lfloor m\tilde b/k\rfloor \wedge (mn)^{\frac{1}{2\alpha+1}}}{nk} + \frac{\bigg(\lfloor m\tilde b/k\rfloor \wedge (mn)^{\frac{1}{2\alpha+1}}\bigg)\delta^2}{3k} 
\end{align*}
where $E_{ij}$ denotes the uniform error introduced by quantizing $X_{ij}$ 
and we have used the fact that they are mutually independent and independent to $X_{ij}$.
Also recall the definitions of $\delta, k, \tilde b$ as appearing in the algorithm. 
Therefore, we have
\begin{align*}
\E_\theta\left[\|\hat\theta - \theta\|^2\right]  \leq \frac{\lfloor m\tilde b/k\rfloor \wedge (mn)^{\frac{1}{2\alpha+1}}}{nk} + \frac{\left(\lfloor m\tilde b/k\rfloor \wedge (mn)^{\frac{1}{2\alpha+1}}\right)\delta^2}{3k}
+ \frac{c^2}{\left(\lfloor m\tilde b/k\rfloor \wedge (mn)^{\frac{1}{2\alpha+1}}\right)^{2\alpha} }.
\end{align*}
Now we analyze the risk for the three regimes respectively. 

In the first regime where $mb\leq n^{\frac{1}{2\alpha+1}}$, we have
\begin{align*}
k &= \left(\lfloor (mb)^{\frac{2\alpha+1}{2\alpha+2}}n^{-\frac{1}{2\alpha+2}}\rfloor \vee 1\right) \wedge m = 1\\
\delta &= \max\left\{(mb)^{-\frac{2\alpha+1}{2}}, n^{-\frac{1}{2}}\right\} = (mb)^{-\frac{2\alpha+1}{2}},
\end{align*}
and consequently
\begin{align*}
\E_\theta\left[\|\hat\theta - \theta\|^2\right] 
& \leq \frac{m\tilde b}{n} + \frac{m\tilde b}{3(mb)^{2\alpha+1}} + \frac{c^2}{(m\tilde b)^{2\alpha} } \\
& = O\left((mb)^{-2\alpha}\log (mb)\right).
\end{align*}

In the second regime where $b \leq (mn)^{\frac{1}{2\alpha+1}}$ and $mb\geq n^{\frac{1}{2\alpha+1}}$, we have
\begin{align*}
k &= \left(\lfloor (mb)^{\frac{2\alpha+1}{2\alpha+2}}n^{-\frac{1}{2\alpha+2}}\rfloor \vee 1\right) \wedge m
 = \lfloor (mb)^{\frac{2\alpha+1}{2\alpha+2}}n^{-\frac{1}{2\alpha+2}}\rfloor,\\
\delta &= \max\left\{(mb)^{-\frac{2\alpha+1}{2}}, n^{-\frac{1}{2}}\right\} = n^{-\frac{1}{2}},
\end{align*}
and it follows that
\begin{align*}
\E_\theta\left[\|\hat\theta - \theta\|^2\right] 
\leq \frac{4\lfloor m\tilde b/k\rfloor}{3nk} + \frac{c^2}{\lfloor m\tilde b/k\rfloor^{2\alpha} }
= O\left((mbn)^{-\frac{\alpha}{\alpha+1}}(\log n)^{2\alpha}\right).
\end{align*}

In the last regime where $b\geq (mn)^{\frac{1}{2\alpha+1}}$, we have
\begin{align*}
k &= \left(\lfloor (mb)^{\frac{2\alpha+1}{2\alpha+2}}n^{-\frac{1}{2\alpha+2}}\rfloor \vee 1\right) \wedge m = m,\\
\delta &= \max\left\{(mb)^{-\frac{2\alpha+1}{2}}, n^{-\frac{1}{2}}\right\} = n^{-\frac{1}{2}},
\end{align*}
and then
\begin{align*}
\E_\theta\left[\|\hat\theta - \theta\|^2\right] 
&\leq \frac{(mn)^{\frac{1}{2\alpha+1}}}{mn} + \frac{(mn)^{\frac{1}{2\alpha+1}}}{3mn} + \frac{c^2}{\tilde b^{2\alpha} }\\
&= \begin{cases}
O\left((mn)^{-\frac{2\alpha}{2\alpha+1}}\right) &\text{if } b \geq (mn)^{\frac{1}{2\alpha+1}}\log n\\
O\left((mn)^{-\frac{2\alpha}{2\alpha+1}}(\log n)^{2\alpha}\right) &\text{otherwise}
\end{cases}.
\end{align*}

One interesting direction for future work is to study 
adaptivity in distributed estimation.
An adaptive protocol $(\Pi, \hat\theta)$ satisfies 
\[
\liminf_{n\to\infty}\frac{\sup_{\theta\in\Theta(\alpha, c)}\mathbb E\left[\|\hat\theta(\Pi) - \theta\|^2\right]}{\inf_{(\Pi, \check\theta)\in\Gamma(m, b)}\sup_{\theta\in\Theta(\alpha, c)}\mathbb E\left[\|\check\theta(\Pi) - \theta\|^2\right]}<\infty,
\]
for (almost) all $\alpha$ and $c$. 
That is, the protocol should be minimax optimal for all $\theta$ without prior knowledge of the parameter space in which $\theta$ resides. 
While we had conjectured that this may not be possible with only one round of
communication, \citet{szabo2018adaptive} recently 
developed an adaptive estimator using a modification of Lepski's method.

A second interesting direction for future work is distributed estimation of other functionals.
For instance, one might study the sum of squares (or $\ell_2$ norm) of the mean of a normal random vector. 
It would be of interest to understand the minimax risk of the norm of
the mean in a distributed setting, and to develop optimal distributed
protocols for this functional.

Finally, other nonparametric problems should be considered in a
distributed estimation setting. For example, it will be interesting to
study nonparametric estimation of functions with varying smoothness
(e.g., over Besov bodies), and with shape constraints such as
monotonicity and convexity.

\section*{Acknowledgment}
Research supported in part by ONR grant N00014-12-1-0762 and NSF grant DMS-1513594.

\bibliography{distributed_nonparametric_estimation}

\begin{thebibliography}{20}
\providecommand{\natexlab}[1]{#1}
\providecommand{\url}[1]{\texttt{#1}}
\expandafter\ifx\csname urlstyle\endcsname\relax
  \providecommand{\doi}[1]{doi: #1}\else
  \providecommand{\doi}{doi: \begingroup \urlstyle{rm}\Url}\fi

\bibitem[Battey et~al.(2015)Battey, Fan, Liu, Lu, and
  Zhu]{battey2015distributed}
Battey, Heather, Fan, Jianqing, Liu, Han, Lu, Junwei, and Zhu, Ziwei.
\newblock Distributed estimation and inference with statistical guarantees.
\newblock \emph{arXiv preprint arXiv:1509.05457}, 2015.

\bibitem[Berger et~al.(1996)Berger, Zhang, and Viswanathan]{berger1996ceo}
Berger, Toby, Zhang, Zhen, and Viswanathan, Harish.
\newblock The {CEO} problem.
\newblock \emph{IEEE Trans.\ Inform.\ Theory}, 42\penalty0 (3):\penalty0
  887--902, 1996.

\bibitem[Blanchard \& M{\"u}cke(2016)Blanchard and
  M{\"u}cke]{blanchard2016parallelizing}
Blanchard, Gilles and M{\"u}cke, Nicole.
\newblock Parallelizing spectral algorithms for kernel learning.
\newblock \emph{arXiv preprint arXiv:1610.07487}, 2016.

\bibitem[Braverman et~al.(2016)Braverman, Garg, Ma, Nguyen, and
  Woodruff]{braverman2016communication}
Braverman, Mark, Garg, Ankit, Ma, Tengyu, Nguyen, Huy~L., and Woodruff,
  David~P.
\newblock Communication lower bounds for statistical estimation problems via a
  distributed data processing inequality.
\newblock In \emph{Proceedings of the Forty-eighth Annual ACM Symposium on
  Theory of Computing}, STOC '16, pp.\  1011--1020, 2016.

\bibitem[Brown \& Low(1996)Brown and Low]{brown1996asymptotic}
Brown, Lawrence~D and Low, Mark~G.
\newblock Asymptotic equivalence of nonparametric regression and white noise.
\newblock \emph{Ann.\ Statist.}, 24\penalty0 (6):\penalty0 2384--2398, 1996.

\bibitem[Chang et~al.(2017)Chang, Lin, and Zhou]{chang2017distributed}
Chang, Xiangyu, Lin, Shao-Bo, and Zhou, Ding-Xuan.
\newblock Distributed semi-supervised learning with kernel ridge regression.
\newblock \emph{Journal of Machine Learning Research}, 18\penalty0
  (46):\penalty0 1--22, 2017.

\bibitem[Diakonikolas et~al.(2017)Diakonikolas, Grigorescu, Li, Natarajan,
  Onak, and Schmidt]{diakonikolas2017communication}
Diakonikolas, Ilias, Grigorescu, Elena, Li, Jerry, Natarajan, Abhiram, Onak,
  Krzysztof, and Schmidt, Ludwig.
\newblock Communication-efficient distributed learning of discrete
  distributions.
\newblock In \emph{Advances in Neural Information Processing Systems}, pp.\
  6394--6404, 2017.

\bibitem[Fan et~al.(2017)Fan, Wang, Wang, and Zhu]{fan2017distributed}
Fan, Jianqing, Wang, Dong, Wang, Kaizheng, and Zhu, Ziwei.
\newblock Distributed estimation of principal eigenspaces.
\newblock \emph{arXiv preprint arXiv:1702.06488}, 2017.

\bibitem[Johnstone(2017)]{johnstone2017gaussian}
Johnstone, Iain~M.
\newblock Gaussian estimation: Sequence and wavelet models.
\newblock Unpublished manuscript, 2017.

\bibitem[Lee et~al.(2017)Lee, Liu, Sun, and Taylor]{lee2017communication}
Lee, Jason~D, Liu, Qiang, Sun, Yuekai, and Taylor, Jonathan~E.
\newblock Communication-efficient sparse regression.
\newblock \emph{Journal of Machine Learning Research}, 18\penalty0
  (5):\penalty0 1--30, 2017.

\bibitem[Nussbaum(1996)]{nussbaum1996asymptotic}
Nussbaum, Michael.
\newblock Asymptotic equivalence of density estimation and {G}aussian white
  noise.
\newblock \emph{Ann.\ of Statist.}, pp.\  2399--2430, 1996.

\bibitem[Shamir(2014)]{shamir2014fundamental}
Shamir, Ohad.
\newblock Fundamental limits of online and distributed algorithms for
  statistical learning and estimation.
\newblock In \emph{Proceedings of the 27th International Conference on Neural
  Information Processing Systems}, NIPS'14, pp.\  163--171, 2014.

\bibitem[Shang \& Cheng(2017)Shang and Cheng]{shang2017computational}
Shang, Zuofeng and Cheng, Guang.
\newblock Computational limits of a distributed algorithm for smoothing spline.
\newblock \emph{The Journal of Machine Learning Research}, 18\penalty0
  (1):\penalty0 3809--3845, 2017.

\bibitem[Szabo \& van Zanten(2018)Szabo and van Zanten]{szabo2018adaptive}
Szabo, Botond and van Zanten, Harry.
\newblock Adaptive distributed methods under communication constraints.
\newblock \emph{arXiv preprint arXiv:1804.00864}, 2018.

\bibitem[Tsybakov(2008)]{tsybakov2008introduction}
Tsybakov, Alexandre~B.
\newblock \emph{Introduction to Nonparametric Estimation}.
\newblock Springer Series in Statistics, 1st edition, 2008.

\bibitem[Viswanathan \& Berger(1997)Viswanathan and
  Berger]{viswanathan1997quadratic}
Viswanathan, H. and Berger, T.
\newblock The quadratic {G}aussian ceo problem.
\newblock \emph{IEEE Transactions on Information Theory}, 43\penalty0
  (5):\penalty0 1549--1559, Sep 1997.

\bibitem[Wang et~al.(2010)Wang, Chen, and Wu]{wang2010sum}
Wang, Jia, Chen, Jun, and Wu, Xiaolin.
\newblock On the sum rate of {G}aussian multiterminal source coding: {N}ew
  proofs and results.
\newblock \emph{IEEE Transactions on Information Theory}, 56\penalty0
  (8):\penalty0 3946--3960, 2010.

\bibitem[Zhang et~al.(2013{\natexlab{a}})Zhang, Duchi, Jordan, and
  Wainwright]{zhang2013information}
Zhang, Yuchen, Duchi, John, Jordan, Michael~I, and Wainwright, Martin~J.
\newblock Information-theoretic lower bounds for distributed statistical
  estimation with communication constraints.
\newblock In \emph{Advances in Neural Information Processing Systems}, pp.\
  2328--2336, 2013{\natexlab{a}}.

\bibitem[Zhang et~al.(2013{\natexlab{b}})Zhang, Duchi, and
  Wainwright]{zhang2013divide}
Zhang, Yuchen, Duchi, John, and Wainwright, Martin.
\newblock Divide and conquer kernel ridge regression.
\newblock In \emph{Conference on Learning Theory}, pp.\  592--617,
  2013{\natexlab{b}}.

\bibitem[Zhu \& Lafferty(2017)Zhu and Lafferty]{zhu2017quantized}
Zhu, Yuancheng and Lafferty, John.
\newblock Quantized minimax estimation over {S}obolev ellipsoids.
\newblock \emph{Information and Inference}, 2017.

\end{thebibliography}
\bibliographystyle{icml2018}

\clearpage


\begin{appendix}
\section{Proof of lemmas}

\subsection{Proof of Lemma \ref{lem:supgtrintegral}}
\begin{proof}[Proof of Lemma \ref{lem:supgtrintegral}]
Write
\[
\Theta_\ell(\alpha,c) = \left\{\theta:\sum_{i=1}^\ell i^{2\alpha}\theta_i^2\leq c^2, \theta_i = 0 \text{ for }i\geq \ell+1\right\}\subset\Theta(\alpha, c).
\]
For $\tau \in (0,1)$, write $s_i^2 = (1-\tau)\sigma_i^2$,
and denote by $\pi_\tau(\theta)$ the prior distribution on $\theta$ such that
$\theta_i \sim N(0, s_i^2)$ for $i=1,\dots,\ell$, and $\mathbb P(\theta_i = 0) = 1$ for $i \geq \ell+1$.
For an estimator $\hat\theta$ and its corresponding communication protocol,
we observe that
\begin{align*}
\sup_{\theta\in\Theta(\alpha,c)}\|\hat\theta - \theta\|^2 
&\geq \sup_{\theta\in\Theta_\ell(\alpha,c)}\|\hat\theta - \theta\|^2 \\
&\geq \int_{\Theta_\ell(\alpha,c)}\|\hat\theta - \theta\|^2 \d\pi_\tau(\theta) \\
&\geq I_\tau - r_\tau
\end{align*}
where $I_\tau$ is the integrated risk of the estimator
\[
I_\tau = \int_{\mathbb R^\ell\otimes\{0\}^\infty}\|\hat\theta - \theta\|^2\d\pi_\tau(\theta)
\]
and $r_\tau$ is the residual
\[
r_\tau = \int_{\overline{\Theta(\alpha, c)}}\|\hat\theta - \theta\|^2\d\pi_\tau(\theta)
\]
where $\overline{\Theta(\alpha,c)} = (\mathbb R^\ell\otimes\{0\}^\infty)\backslash\Theta_\ell(\alpha, c)$.
As $\lim_{\tau\to 0}I_\tau = \int_{\Theta}\E_\theta[\|\hat\theta - \theta\|^2] \d\pi(\theta)$, 
it suffices to show that $r_\tau = o(I_\tau)$ as $\ell\to\infty$ for $\tau\in(0,1)$.
Let $B_\ell = \sup_{\theta\in\Theta_\ell(\alpha, c)}\|\theta\|$, 
which is bounded since for any $\theta\in\Theta_\ell(\alpha, c)$
\[
\|\theta\| = \sqrt{\sum_{i=1}^\ell\theta_i^2} = \sqrt{\sum_{i=1}^\ell i^{2\alpha}\theta_i^2} \leq \sqrt{c^2} = c.
\]
We have
\begin{align*}
r_\tau & = \int_{\overline{\Theta_\ell(\alpha, c)}} \E_\theta\left[\|\hat\theta - \theta\|^2\right]\d\pi_\tau(\theta)\\
& \leq 2 \int_{\overline{\Theta_\ell(\alpha, c)}} \left(B_\ell^2 + \|\theta\|^2\right)\d\pi_\tau(\theta)\\
& \leq 2 \left(B_\ell^2\mathbb P\left(\theta\notin\Theta_\ell(\alpha, c)\right) + \left(\mathbb P\left(\theta\notin\Theta_\ell(\alpha, c)\right)\E\left[\|\theta\|^4\right]\right)\right)
\end{align*}
where we have used the Cauchy-Schwarz inequality. Noticing that
\begin{align*}
\E\left[\|\theta\|^4\right] & = \E\left[\left(\sum_{i=1}^\ell\theta_i^2\right)^2\right] \\
& = \sum_{i_1\neq i_2}\E\left[\theta_{i_1}^2\right]\E\left[\theta_{i_2}^2\right] + \sum_{i=1}^\ell\E\left[\theta_i^4\right]\\
& \leq \sum_{i_1\neq i_2} s_{i_1}^2s_{i_2}^2 + 3\sum_{i=1}^\ell s_i^4\\
& \leq 3(\sum_{i=1}^\ell s_i^2)^2 \leq 3 B_\ell^4,
\end{align*}
we obtain
\[
r_\tau \leq 2B_\ell^2\left(\mathbb P\left(\theta\notin\Theta_\ell(\alpha,c)\right) + \sqrt{3\mathbb P\left(\theta\notin\Theta_\ell(\alpha,c)\right)}\right)
\leq 6B_\ell^2\sqrt{3\mathbb P\left(\theta\notin\Theta_\ell(\alpha,c)\right)}.
\]
Thus, we only need to show that $\sqrt{\mathbb P\left(\theta\notin\Theta_\ell(\alpha,c)\right)} = o(I_\tau)$. In fact,
\begin{align*}
\mathbb P\left(\theta\notin\Theta_\ell(\alpha,c)\right)
& = \mathbb P\left(\sum_{i=1}^\ell i^{2\alpha}\theta_i^2 > c^2\right)\\
& = \mathbb P\left(\sum_{i=1}^\ell i^{2\alpha}(\theta_i^2-\E[\theta_i^2]) > c^2 - (1-\tau)\sum_{i=1}^\ell i^{2\alpha}\sigma_i^2\right)\\
& = \mathbb P\left(\sum_{i=1}^\ell i^{2\alpha}(\theta_i^2-\mathbb E[\theta_i^2]) > \tau c^2\right)\\
& = \mathbb P\left(\sum_{i=1}^\ell i^{2\alpha} s_i^2(Z_i^2-1) > \frac{\tau}{1-\tau}\sum_{i=1}^\ell i^{2\alpha}s_i^2\right)
\end{align*}
where $Z_i\sim N(0,1)$.
By Lemma \ref{chisqineq}, we get
\begin{equation*}
\mathbb P(\theta\notin\Theta_\ell (m,c))\leq \exp\left(-\frac{\tau^2}{8(1-\tau)^2}\frac{\sum_{i=1}^\ell i^{2\alpha}s_i^2}{\max_{1\leq i\leq \ell}i^{2\alpha}s_i^2}\right)
=\exp\left(-\frac{\tau^2}{8(1-\tau)^2}\frac{\sum_{i=1}^\ell i^{2\alpha}\sigma_i^2}{\max_{1\leq i\leq \ell}i^{2\alpha}\sigma_i^2}\right).
\end{equation*}
By the assumption that $\frac{\sum_{i=1}^\ell i^{2\alpha}\sigma_i^2}{\max_{1\leq i\leq \ell}i^{2\alpha}\sigma_i^2} = O(\ell)$,
and that $\int_{\Theta}\E_\theta[\|\hat\theta - \theta\|^2] \d\pi(\theta) = O(\ell^\delta)$,
we conclude that $r_\tau = o(I_\tau)$ as $\ell\to\infty$.
\end{proof}

\begin{lemma}[Lemma 3.5 in \citep{tsybakov2008introduction}]\label{chisqineq}
Suppose that $Z_1,\dots,Z_n\sim N(0,1)$ independently. For $t\in(0,1)$ and $\omega_i>0$, $i=1,\dots,n$, we have
\begin{equation*}
\mathbb P\left(\sum_{i=1}^n\omega_i(Z_i^2-1)>t\sum_{i=1}^n\omega_i\right)\leq\exp\left(-\frac{t^2\sum_{i=1}^n\omega_i}{8\max_{1\leq i\leq n}\omega_i}\right).
\end{equation*}
\end{lemma}

\subsection{Proof of Lemma \ref{lem:lowerboundwithprior}}
\begin{proof}[Proof of Lemma \ref{lem:lowerboundwithprior}]
Recall that we have 
$\theta_i \sim N(0, \sigma_i^2)$ and
$X_{ij} | \theta_i \sim N(\theta_i, \varepsilon^2)$
for $i = 1, \dots, \ell$, and $j = 1, \dots, m$.
For convenience, write $\theta = (\theta_1,\dots,\theta_n)$, 
$X_j = (X_{1j},\dots,X_{\ell j})$ and $X = (X_1,\dots, X_j)$.
Suppose that we have a set of encoding functions $\Pi_j: \mathbb R^\ell\to\{1,\dots, M_j\}$
for $j = 1,\dots, m$ satisfying that $\sum_{j=1}^m\log M_j \leq mb$.
Let $W_j = \Pi_j(X_j)$ be the message generated from the $j$th machine,
and write $W = (W_1,\dots, W_m)$.
Furthermore, we write $d_i = \E (\theta_i - \E(\theta_i|W))^2$
and $d_{ij} = \E(X_{ij}|\theta, W_j)^2$.
We then have
\begin{align*}
\sum_{j=1}^m \log M_j & \geq H(W) \\ 
& \geq I(\theta, X; W) \\
& = I(\theta; W) + \sum_{j=1}^mI(X_j; W_j | \theta) \\
& = h(\theta) - h(\theta|W) + \sum_{j=1}^m \left(h(X_j| \theta) - h(X_j|\theta,W)\right) \\
& = \sum_{i=1}^\ell  h(\theta_i) - \sum_{i=1}^\ell  h(\theta_i|\theta_{1:(i-1)}, W)
      + \sum_{j=1}^m \left( \sum_{i=1}^\ell  h(X_{ij}| \theta) - \sum_{i=1}^\ell h(X_{ij}|X_{1:(i-1),j},\theta,W_j)\right) \\
& \geq \sum_{i=1}^\ell  h(\theta_i) - \sum_{i=1}^\ell  h(\theta_i|W)
      + \sum_{j=1}^m \left( \sum_{i=1}^\ell  h(X_{ij}| \theta) - \sum_{i=1}^\ell h(X_{ij}|\theta,W_j)\right) \\
& \geq \sum_{i=1}^\ell  \left(\frac{1}{2}\log\frac{\sigma_i^2}{d_i} + \sum_{j=1}^m\frac{1}{2}\log\frac{\varepsilon^2}{d_{ij}}\right).
\numberthis\label{eqn:sumlogM}
\end{align*}

In order to obtain the relationship between $d_i$'s and $d_{ij}$'s, 
we consider the random vector $Y = \E(\theta|X)$, 
i.e., $Y_i = \E(\theta_i|X)$ for $i = 1,\dots, n$.
In fact, $Y_i$ takes the form 
\[
Y_i = \frac{\frac{1}{\varepsilon^2}}{\frac{1}{\sigma_i^2} + \frac{m}{\varepsilon^2}}\sum_{j=1}^m X_{ij}.
\]

We first calculate the optimal mean squared error of estimating $Y_i$ based on $\theta$ and $W$
\[
\E\left[(Y_i - \E(Y_i|\theta, W))^2\right] 
= \left(\frac{\frac{1}{\varepsilon^2}}{\frac{1}{\sigma_i^2} + \frac{m}{\varepsilon^2}}\right)^2\E\left[\left(\sum_{j=1}^m\left(X_{ij} - \E[X_{ij}|\theta, W_j]\right)\right)^2\right]
= \left(\frac{\frac{1}{\varepsilon^2}}{\frac{1}{\sigma_i^2} + \frac{m}{\varepsilon^2}}\right)^2\sum_{j=1}^m d_{ij}
\]
where we have used the equality that
\begin{align*}
&\E\left[(X_{ij} - \E[X_{ij}|\theta, W_j])(X_{ij'} - \E[X_{ij'}|\theta, W_{j'}])\right] \\
& = \E\left[X_{ij} - \E[X_{ij}|\theta, X_{ij'}, W_j, W_{j'}]\right] \E\left[X_{ij'} - \E[X_{ij'}|\theta, W_{j'}]\right] = 0
\end{align*}
for $j\neq j'$.

We then calculate the mean squared error of best linear estimator of $Y_i$ using $\theta_i$ and $T_i = \E(\theta_i | W)$.
In particular, we search for $\beta_1$ and $\beta_2$ such that
\[
\E\left[(Y_i - \beta_1\theta_i - \beta_2T_i)^2\right]
\]
is minimized. 
Towards that end, we calculate
\[
\E\left[Y_i^2\right] = \left(\frac{\frac{1}{\varepsilon^2}}{\frac{1}{\sigma_i^2} + \frac{m}{\varepsilon^2}}\right)^2\E\left[\left(\sum_{j=1}^mX_{ij}\right)^2\right]
= \left(\frac{\frac{1}{\varepsilon^2}}{\frac{1}{\sigma_i^2} + \frac{m}{\varepsilon^2}}\right)^2\left(m^2\varepsilon^2 + m\sigma_i^2\right)
= \frac{\frac{m}{\varepsilon^2}}{\frac{1}{\sigma_i^2} + \frac{m}{\varepsilon^2}}\sigma_i^2
= \sigma_i^2 - \sigma_0^2
\]
where we write $\sigma_0^2 = \frac{1}{\frac{1}{\sigma_i^2} + \frac{m}{\varepsilon^2}}$ to ease our notation.
In addition, we have
\[
\E\left[\theta_i^2\right] = \sigma_i^2 \text{ and }
\E \left[Y_i\theta_i\right] =\frac{\frac{1}{\varepsilon^2}}{\frac{1}{\sigma_i^2} + \frac{m}{\varepsilon^2}}\E\left[\sum_{j=1}^m\theta_iX_{ij}\right] = \sigma_i^2 - \sigma_0^2.
\]
Furthermore, we notice that since $T_i = \E\left[\theta_i | W\right]$,
\[
\E\left[T_i(\theta_i - T_i)\right] = \E\left[\E\left[T_i(\theta_i - T_i)\right]|W\right] = \E\left[T_i\right] \E\left[T_i - T_i\right] = 0,
\]
and hence
\[
d_i = \E\left[(\theta_i - T_i)^2\right] = \E\left[\theta_i(\theta_i - T_i) - T_i(\theta_i - T_i)\right] = \E\left[\theta_i(\theta_i - T_i)\right] = \E\left[\theta_i^2\right] - \E\left[\theta_iT_i\right],
\]
from which we obtain
\[
\E\left[T_i^2\right] = \E\left[\theta_iT_i\right] = \sigma_i^2 - d_i.
\]
Finally, we have
\[
\E\left[Y_iT_i\right] = \E\left[(\theta_i + (Y_i - \theta_i))T_i\right] = \E\left[\theta_iT_i\right] + \E[Y_i- \theta_i]\E[T_i]  = \E\left[\theta_iT_i\right] = \sigma_i^2 - d_i
\]
where the equality follows from the fact that $\theta_i$ and $\theta_i - Y_i$ are independent.
To sum up, the covariance matrix of $(Y_i, \theta_i, T_i)$ is 
\[
\begin{pmatrix}
\sigma_i^2 - \sigma_0^2 & \sigma_i^2 - \sigma_0^2 & \sigma_i^2 - d_i \\
\sigma_i^2 - \sigma_0^2 & \sigma_i^2 & \sigma_i^2 - d_i \\
\sigma_i^2 - d_i & \sigma_i^2 - d_i & \sigma_i^2 - d_i
\end{pmatrix}.
\]
Getting back to $\beta_1$ and $\beta_2$, they should satisfy
\[
\E\left[\theta_i(Y_i - \beta_1\theta_i - \beta_2T_i)\right] = 0,\quad \E\left[T_i(Y_i - \beta_1\theta_i - \beta_2T_i)\right] = 0.
\]
Solving the equations, we get
\[
\beta_1 = \frac{d_i - \sigma_0^2}{d_i}, \quad \beta_2 = \frac{\sigma_0^2}{d_i},
\]
and 
\[
\E\left[(Y_i - \beta_1\theta_i - \beta_2T_i)^2\right] = \sigma_0^2 - \frac{\sigma_0^4}{d_i}.
\]

Since conditional means minimize mean squared errors, we have
\[
\E\left[(Y_i - \E(Y_i|\theta, W))^2\right] \leq \E\left[(Y_i - \beta_1\theta_i - \beta_2T_i)^2\right]
\]
and therefore, 
\[
\left(\frac{\frac{1}{\varepsilon^2}}{\frac{1}{\sigma_i^2} + \frac{m}{\varepsilon^2}}\right)^2\sum_{j=1}^m d_{ij} \leq 
\frac{1}{\frac{1}{\sigma_i^2} + \frac{m}{\varepsilon^2}} - \left(\frac{1}{\frac{1}{\sigma_i^2} + \frac{m}{\varepsilon^2}}\right)^2\frac{1}{d_i},
\]
which gives
\[
\sum_{j=1}^m d_{ij} \leq \varepsilon^4\left(\frac{1}{\sigma_i^2} + \frac{m}{\varepsilon^2}-\frac{1}{d_i}\right).
\]
Now we plug this into \eqref{eqn:sumlogM}, and obtain by applying Jensen's inequality that
\begin{align*}
mb & \geq \sum_{i=1}^\ell\left(\frac{1}{2}\log\frac{\sigma_i^2}{d_i} + \sum_{j=1}^m\frac{1}{2}\log\frac{\varepsilon^2}{d_{ij}}\right)\\
& \geq \sum_{i=1}^\ell\left(\frac{1}{2}\log\frac{\sigma_i^2}{d_i} + \frac{m}{2}\log\frac{\varepsilon^2}{\frac{1}{m}\sum_{j=1}^md_{ij}}\right)\\
& \geq \sum_{i=1}^\ell\left(\frac{1}{2}\log\frac{\sigma_i^2}{d_i} + \frac{m}{2}\log\frac{\frac{m}{\varepsilon^2}}{\frac{1}{\sigma_i^2} + \frac{m}{\varepsilon^2}-\frac{1}{d_i}}\right),
\end{align*}
which completes the proof.
\end{proof}
\end{appendix}

\end{document}